\documentclass[journal]{IEEEtran}
\ifCLASSOPTIONcompsoc
\usepackage[nocompress]{cite}
\else
\usepackage{cite}
\fi
\ifCLASSINFOpdf
\else
\fi
\usepackage{multirow,palatino, epsfig, latexsym, algorithmic, algorithm, epstopdf, amsmath, amssymb, color, amsthm,graphicx,enumitem, hyperref}

\newcommand{\ignore}[1]{}

\DeclareMathOperator*{\argmax}{argmax}

\newtheorem{myTheorem}{Theorem}
\newtheorem{myRemark}{Remark}

\newtheorem{myDefinition}{Definition}

\usepackage[utf8]{inputenc} 
\usepackage[T1]{fontenc}    
\usepackage{hyperref}       
\usepackage{url}            
\usepackage{booktabs}       
\usepackage{amsfonts}       
\usepackage{nicefrac}       
\usepackage{microtype}      
\usepackage{amsmath}
\usepackage{algorithmic}
\usepackage{algorithm}
\usepackage{multirow,palatino, epsfig, latexsym, algorithmic, algorithm, epstopdf, amsmath, amssymb, color, amsthm,graphicx,enumitem, hyperref}

\begin{document}

\title{Self Punishment and Reward Backfill for Deep Q-Learning}

%

\author{Mohammad Reza Bonyadi, Rui Wang, Maryam Ziaei
	
	\thanks{Mohammad Reza Bonyadi (rezabonyadi@cmicrosoft.com, rezabny@gmail.com) is with Microsoft Development Center Norway (MDCN), Norway, Trondheim. Rui Wang is with Turnto Health. Maryam Ziaei is with the Kavli Institute for Systems Neuroscience, Norwegian University of Science and Technology (NTNU), Trondheim, Norway, Jebsen Centre for Alzheimer’s Diseases, Norwegian University of Science and Technology (NTNU), Trondheim, Norway, and with Queensland Brain Institute, University of Queensland, Brisbane, Australia.}
}

\IEEEtitleabstractindextext{%
	\begin{abstract}
	Reinforcement learning agents learn by encouraging behaviours which maximize their total reward, usually provided by the environment. In many environments, however, the reward is provided after a series of actions rather than each single action, leading the agent to experience ambiguity in terms of whether those actions are effective, an issue known as the credit assignment problem. In this paper, we propose two strategies inspired by behavioural psychology to enable the agent to intrinsically estimate more informative reward values for actions with no reward. The first strategy, called self-punishment (SP), discourages the agent from making mistakes that lead to undesirable terminal states. The second strategy, called the rewards backfill (RB), backpropagates the rewards between two rewarded actions. We prove that, under certain assumptions and regardless of the reinforcement learning algorithm used, these two strategies maintain the order of policies in the space of all possible policies in terms of their total reward, and, by extension, maintain the optimal policy. Hence, our proposed strategies integrate with any reinforcement learning algorithm that learns a value or action-value function through experience. We incorporated these two strategies into three popular deep reinforcement learning approaches and evaluated the results on thirty Atari games. After parameter tuning, our results indicate that the proposed strategies improve the tested methods in over 65 percent of tested games by up to over 25 times performance improvement.
	\end{abstract}
	
}

\maketitle
\IEEEdisplaynontitleabstractindextext
\IEEEpeerreviewmaketitle

\section{Introduction}
A reinforcement learning (RL) agent~\cite{sutton1988learning} aims to make optimal decisions in an environment from which it receives rewards for its actions. The agent optimizes a relationship (policy) between actions and rewards, which is later used to make plausible actions. One of the frequently used techniques to form this relationship is through learning the forward quality of actions in any given state, known as Q-Learning~\cite{watkins1992q}. Recently, Deep Q-Learning Networks (DQN)~\cite{mnih2015human} have been introduced and successfully applied to complex reinforcement learning tasks such as Atari 2600 and Star Craft II~\cite{vinyals2017starcraft}. 

There exist a range of environments in which there is no reward (or zero reward) corresponding to majority of actions, posing a significant challenge to RL algorithms \cite{sutton2018reinforcement,daley2019reconciling}. In this paper, we introduce two strategies, namely the Self punishment (SP) and reward backfill (RB), which provide additional reward signals for the agent to learn from. Inspired by the operant conditioning \cite{skinner1951teach}, SP provides undesirable values (punishment) for agent's undesirable actions (e.g., actions which lead to a terminal state) to discourage it from making mistakes. This strategy provides additional information about the state and enables the agent to find better gradient trajectories towards avoiding a loss and, consequently, improve faster. Inspired by the "clicker training" strategy used for animal training \cite{skinner1951teach,pryor1999don}, the RB strategy backpropagates non-zero rewards received from the environment (or from SP strategy) to previous state-action pairs. This strategy provides additional information on how much each state-action with zero reward contributes to a state-action with non-zero reward in the future. 

The paper is organized as follows: Section \ref{sec:background} provides an overview of reinforcement learning as well as the related works. Section \ref{sec:propsoed} describes our approaches and Section \ref{sec:experiment} presents experimental results and discussions. We conclude our work with an outlook to the future work in Section \ref{sec:conclusion}.

\section{Background}
\label{sec:background}
A brief review on reinforcement learning, Q-learning, deep architectures, and credit assignment problem is provided in this section.
\subsection{Reinforcement and Q-Learning}
\label{sec:rlandq}
We define a RL environment (sometimes referred to as a RL problem in this paper) by a tuple $ (S, A, e, r) $, $ S $ the set of states an agent can be in, $ A $ the set of actions an agent can choose from, $ e(s, a):S \times A \to S $ a function that takes the agent from state $ s \in S $ to $ s' \in S $ given the action $ a \in A $, and $ r(s, a): S \times A \to \mathbb{R} $ is the reward provided for the agent given the state $ s $ and the action $ a $. A RL agent (also called RL algorithm) interacts with an environment, such as an Atari emulator, to learn how to behave optimally. The learning is done by optimizing a (usually deterministic) policy $\pi(s): S \to A$ in a way that the aggregated rewards onwards (expected value if non-deterministic), denoted by $ v(\pi, s) $, is maximized. In this paper, we assume $S$ and $A$ are finite sets and an episode has a finite horizon, i.e., there exists a terminal state for any policy and any initial state. We further assume that the environment is fully observable, deterministic, and dynamic. With $ v(\pi, s_0) $ the total reward received by following the policy $ \pi $ in a RL environment and given the initial state $ s_0 $, the existence of the $ \pi^* $, the policy with maximum possible total reward, necessitates that the set $ V $, defined by $ \{v(\pi, s_0)|all~possible~\pi\} $, is an "ordered" set under the operator "$ \le $".

Value-based reinforcement learning algorithms define $Q(s, a): S \times A \to \mathbb{R}$ that represents the total discounted reward the agent may receive by taking action $a$ at the state $s$ onwards following a policy $ \pi $. If the best $ Q $, called $ Q^* $, is known then the optimal policy, $ \pi^* $, can be represented by $ \argmax_{a \in A} Q^*(s,a) $ for any given $ s $. One of the most commonly used methods to estimate $ Q $ is the $ n $-step temporal difference (TD) update \cite{sutton2018reinforcement}:
\begin{equation}
Q(s_t, a_t) \longleftarrow Q(s_t, a_t) + \alpha (R^{(n)}_t - Q(s_t, a_t))
\label{tdvalue}
\end{equation}
where $ R^{(n)}_t $ is referred to as \textit{TD target}, defined by:
\begin{equation}
R^{(n)}_t= \gamma^n \max_{a \in A} Q(s_{t+n}, a)+\sum_{i=0}^{n-1} \gamma^{i} r(s_{t+i}, a_{t+i})
\end{equation}
where $ n \in \{1, 2, ..., T-t\} $ and $ \gamma $ is the discount factor. TD target is set to $ r(s_t, a_t) $ if $ s_{t+1} $ is a terminal state, $ T $. Q-Learning~\cite{sutton2018reinforcement,watkins1992q} is a special case of $ n $-step TD with $ n=1 $. 

\ignore{
Because calculation of TD for $ n>1 $ step is computationally expensive, approximation methods have been implemented for $ R^{(n)}_t $. A popular choice for approximation of $ R^{(n)}_t $ is called the $ \lambda $-return ($ R^{(\lambda)}_t $), estimated by the following recursion \cite{peng1994incremental}:
\begin{equation}
R^\lambda_t = R^{(1)}_t+\gamma\lambda[R^\lambda_{t+1}-\max_{a' \in A}Q(s_{t+1}, a')]
\end{equation}
where $\lambda \in [0, 1]$ is a parameter.
}
\subsection{Deep Q Network (DQN)}
\label{sec:DQN} 
In order to interact with the environment through raw vision information, Deep Q Network (DQN) parameterizes Q function as $Q(s, a; \theta)$ \cite{mnih2015human}. The parameter set $\theta$ is a Convolutional Neural Network (CNN) that transforms a given state $s$ and an action $a$ to their Q values. The DQN agent takes $n$ frames (with a skip parameter $k$ that only selects observation on every $k^{th}$ frame) as one state that is subsequently processed by the CNN with $n$ input channels. Given a state, the network decides on the action with probability $1-\epsilon$ (or performs a random action with probability $\epsilon$ for exploration purposes), and the environment provides the reward and the next state accordingly. This interaction continues until a terminal state emerges. The value of $\epsilon$ is closer to $1.0$ at the earlier stages of the learning process to encourage exploration. Each interaction produces a transition tuple $<s, a, r(s, a), s'>$ that is stored at the replay memory. At each step, a batch of samples is randomly selected from these memorized samples to train the network. The learning objective is to minimize a loss function $ L_t(\theta)=\mathbb{E}[(Q(s_t, a_t; \theta) - R^{(1)}_t)^2]  $, where $R^{(1)}_t = r(s_t, a_t) + \gamma \max_{a \in A} Q(s_{t+1}, a; \theta)$ is the TD target\cite{sutton2018reinforcement,peng1994incremental,watkins1992q}. 

DQN uses a greedy policy to approximate $Q$ defined as $\pi_Q(s) = \argmax_{a \in A} Q(s,a; \theta)$. While this is correct if $ Q $ is equal to $ Q^* $, this approximation may lead to overestimation of the $ Q $ values for non-optimal actions which are regularly selected \cite{van2016deep}. This issue stems from dependency between weights used to calculate the target and the Q-values. Double DQN~\cite{van2016deep} (D2QN) addresses this problem by employing two networks to decouple the action selection and the target value, where the Q-network decides the best action and the target network computes the TD target as $r(s, a) + \gamma Q(s', \argmax_{a' \in A} Q(s', a'; \theta); \theta^-)$. Dueling Double DQN~\cite{wang2015dueling} (D3QN) decomposes the Q-value estimation function into \textit{state-value} function and \textit{action advantage} function, updated separately. The Q-values are then estimated as the sum of state-value function and action advantage function. Dueling architecture improves the learning in two ways. First, the estimation of state-value instead of Q-value explicitly indicates more valuable states. Second, the action advantage function provides additional information to the estimation that decreases noise stemming from overestimation as two functions make independent estimations. DQN was further extended by other approaches such as the use of a noisy model with trainable noise parameters \cite{fortunato2017noisy}. Again, TD target in these variants is set to $ r(s_t, a_t) $ if $ s_{t+1} $ is a terminal state~\cite{mnih2015human, van2016deep, wang2015dueling}.

\subsection{Credit assignment problem}
Not all actions receive a non-zero reward in many RL environments. Nevertheless, a sequence of these mostly zero-rewarded actions may lead the agent to a state for which a reward is received. It is, however, not straightforward to estimate the contribution of these zero-rewarded actions to a future reward, an issue known as the \textit{credit assignment} problem \cite{sutton2018reinforcement}. The severity of this issue increases with the sparsity of non-zero rewards. Q-learning update rule ($ n=1 $ in $ n $-step TD update), in particular, is prone to the credit assignment problem as it considers a rather short sight of actual future rewards \cite{sutton2018reinforcement,daley2019reconciling}, leading to a large bias and a longer learning time. In the context of DQN and environments with sparse rewards, some of the states in the replay memory would be associated with zero-rewarded actions. When the rewards are sparse, the value of $ R^{(n)}_t $ is equal to $ \gamma^n \max_{a \in A} Q(s_{t+n}, a) $ in most cases, especially if $ n=1 $ (i.e., Q-learning), which encourages the loss function to converge to zero quickly. One obvious approach to address the credit assignment problem in a sparse reward space is to increase the value of $ n $, which in turn increases chances of having non-zero actual rewards in the $ R^{(n)}_t $ \cite{barto1983neuronlike}. Incorporating this strategy in DQN is, however, extremely challenging and would require heavy calculations \cite{daley2019reconciling}. There have been other approaches targeting this issue for DQN including defining prioritization for replay memory items~\cite{schaul2015prioritized}, curiosity-driven Exploration ~\cite{pathak2017curiosity}, and hybrid supervised-reinforcement~\cite{vevcerik2017leveraging,nair2018overcoming,hester2018deep}. 

\subsection{Reward shaping}

Reward shaping~\cite{ng1999policy,ng2003shaping,gu2017deep}, the approach that our proposed methods are most related to, replaces the original reward provided by the environment, $r(s, a)$, by an estimation, $z(s, a, s')=r(s, a)+f(s, a, s')$, which not only represents the reward but also assigns rewards to sub-goals by using $ f: S \times A \times S \to \mathbb{R} $. The method can potentially address the credit assignment problem by introducing a function $ f(s, a, s') $ that returns non-zero values when $ r(s, a) $ is zero. It has been proven~\cite{ng1999policy} that if there exists a function $\Phi: S \to \mathbb{R}$ such that $f(s, a)=\gamma \Phi(s') - \Phi(s)$, $\gamma \in [0, 1]$, then the optimal Q function with the modified reward is the same as the optimal Q function using the original reward function. Such reward shaping function is called the \textit{potential-based reward shaping} function. It has been also proven that this form of reward shaping is both necessary and sufficient to guarantee that the optimality of policies is preserved when Q-learning is used. The proof provided in~\cite{ng1999policy} mainly relies on the recursive nature of the Q-learning formulation. 

A reward shaping function was proposed in \cite{dong2020principled} based on a Lyapunov function. The method replaced the reward at each state with a weighted average of the current reward and the reward in the next state, received if the best action (provided by the current Q function) is taken. It was proven that the proposed shaping function is an instance of the potential-based reward shaping. Results showed improvement when this shaping function was used. Another approach based on similarity between decision trajectories was proposed in \cite{zhu2020meta}. In that method, if the current trajectory was "similar" (based on a distance measure) to high quality ones from the memory then the reward given to the agent was increased. The increment was calculated based on a weighted average over most successful trajectories (similar to the k-nearest neighbor) stored in the memory. With the length of the stored trajectories was limited to 100 (i.e., storing the last 100 actions as a trajectory) and the size of the memory to 20 (i.e., storing the top 20 trajectories), it was shown that the proposed method offers improvement over D2QN on some RL problems.

Leveraging pre-trained reward models and improving them during a given task has been a popular way to implement reward shaping. An approach based on a Bayesian framework was proposed in \cite{marom2018belief} in which the "belief" about the reward distribution of the environment was used to shape the reward. The original reward distribution belief was updated along optimization of the RL policy by collecting more evidence about each transition. The mean of the belief distribution was then used in a Q-learning framework, replacing the environment reward at each state. A meta-learning approach, implemented by a neural network, was introduced in \cite{zou2019reward}. The idea was to train the meta-learner on a wide variety of tasks to learn estimating the value of a given state and then use the pre-trained meta-learner in a new task to estimate the value of each given state, added to the reward provided by the environment at that state. As the distribution of the states in the tasks on which the meta-learner was trained might not represent the distribution of the states in a new task, it was proposed to use the pre-trained meta-learner as a prior and calculate a posterior during learning of the new task. The method was most efficient on tasks that had shared state space, making it less effective on tasks with a wide variate of state spaces such as the ones that use vision inputs. In \cite{gimelfarb2018reinforcement}, it was proposed to use a Bayesian Model Combination framework to merge multi-expert opinion, available through external inputs, at each state in an RL task to make the final decision. It was assumed that each expert opinion is a potential-based reward shaping function and it was proven that the final combination is also a potential-based reward shaping function, preserving the optimality of the solutions. Results showed that the method could automatically leverage the best expert opinion more, without being aware which opinion is actually the best. The final performance is dependent on the provided expert models to combine.

\ignore{
	Hindsight reward shaping \cite{de2020hindsight} implements reward shaping idea based on back-propagating the final reward achieved at the terminal state, decayed exponentially, and add that to the original rewards provided for each state by the environment. It was been shown that the proposed strategy can be effective for some Atari environments. 
It is, however, not clear whether this reshaping form is still necessary or sufficient if other methods for finding optimal policy, such as policy gradient and TD with $ n>1 $, are used.
	
Recently, it has been proposed \cite{daley2019reconciling} to use the $\lambda$-return estimation in \cite{peng1994incremental} to enable this incorporation. This approach, however, requires refreshing the $\lambda$-returns in the memory as the Q function is updated, which is an expensive calculation because the Q-function is a deep neural network. To address this issue, \cite{daley2019reconciling} proposed to use a cache in which there are only a subset of items in the replay memory and they include their $\lambda$-return estimation and use that cache to create the minibatches. One issue with this approach is that it limits the learning algorithm in terms of which items can be used for training. This approach also breaks the independent and identically distributed data (i.i.d) assumption, essential for DQN. Finally, this approach still requires refreshing the Q values in the cache when the Q function is updated.}

\section{Proposed Approach}
\label{sec:propsoed}
We propose two strategies, namely the self punishment (SP) and reward backfill (RB), to deal with the credit assignment problem. Both strategies provide easy to implement approaches to establish extra reward information which lead to finding better solutions quicker. 

\subsection{Self Punishment (SP)}
Reinforcement learning has a root in a framework of animal psychology called the operant conditioning \cite{skinner1951teach}. Operant conditioning argues that four main types of feedback from the environment can be used to train animals to shape a behaviour, as follows: 
\begin{itemize}
	\item Positive reinforcement by which a favorable outcome is presented after a desirable behavior. 
	\item Negative reinforcement by which an unfavorable outcome is removed after a desirable behavior. 
	\item Positive punishment by which an unfavorable outcome is presented after an undesirable behavior.
	\item Negative punishment by which a favorable outcome is removed after an undesirable behavior.
\end{itemize}
By using these four signals, behaviors are likely to be repeated if followed by a reward (pleasant consequences) or less likely to be repeated if followed by a punishment (unpleasant consequences). The rewards provided for the DQN agents, discussed in Section \ref{sec:background}, are received from the environment with no major alternations. Hence, as majority of environments reward the desirable states only\footnote{Note that some environments such as "Pong" provide negative reward values when the agent losses a game. Majority of environments, however, tested with DQN variants do not provide signals for loosing or for a terminal state.}, the agent does not receive other types of operant conditioning signals. The lack of positive punishment, in particular, makes it impossible for the agent to distinguish between a neutral state (a state with no reward, but not necessarily a bad state) and an undesirable one. Even in the case of a terminal state where a life is lost, the target value is calculated by $ r(s, a) $. As $ r(s, a) $ is usually zero if the state is an undesirable terminal one, this reward does not provide any information for the agent to avoid this state. 

In this paper, we propose a strategy by which the agent shapes an intrinsic auxiliary reward to "self-punish" at a terminal state, without the environment explicitly providing that value. In particular, if the next state after the current state by taking action $ a $ is terminal then the agent modifies the received reward to $ r(s, a) - p $, where $ p \in \mathbb{R^+} $ is a constant. Otherwise, the agent does not modify the reward received from the environment.
\ignore{
rather than $ r(s, a) $\footnote{Note that not all terminal states are considered to be undesirable. For a desirable terminal state (e.g., game won), the reward $ r(s, a) $ is unlikely to be zero, hence, the proposed strategy only shifts the total reward.}. We propose to set $ \psi(s, a)=-p $, , when the next state after $ s $, taking action $ a $, is a terminal state, and zero otherwise. 
\begin{equation}
\label{eq:PP}
\psi(s, a) =
\begin{cases}
-p, & \text{if } s' \text{ is a terminal state} \\
0, & \text{otherwise}
\end{cases}       
\end{equation}
where $ p \in \mathbb{R^+} $ and $s'$ is the next state after $s$ by taking action $a$. 
}
We set the value of $p$ experimentally in Section \ref{sec:experiment}. One should note that not all terminal states are undesirable. If the terminal state is an undesirable one then the SP strategy reduces the reward, which would make the distinction between a neutral and an undesirable terminal state more apparent. For a desirable terminal state (e.g., game won, task  done successfully), on the other hand, the environment is unlikely to provide a zero reward and expected to provide a large positive reward. In this case, the proposed SP strategy only slightly shifts the total reward, which would not impact the overall reward as it is applied to all terminal states (see the proof below).  

Let us exemplify how SP strategy may improve the RL agent's learning ability in the game of Breakout in Atari 2600. Let us first assume that SP is not used in Breakout and consider a situation where the ball is very close to the right side of the moving paddle but outside of the paddle and approaching. In this setting, although choosing to "move right" or to "move left" leads to a zero immediate reward by the environment, the latter leads to terminating the game while the former leads to continuing the game. The zero reward provided by the environment makes it impossible for the agent to distinguish between the terminal state and a zero reward state, hence, the game may be lost without any signal for the agent why it happened. In contrast, if the SP strategy was used, an additional immediate signal for the agent was provided to separate a terminal decision from a "neutral" one, which would guide the agent to avoid this mistake from the early stages and improve faster. One should note that, if the Q function is accurate ($ Q=Q^* $, where $ Q^* $ is the optimal Q-function) then the Q value of the action "right" would be larger than any other action, which would lead to choosing the action "right". This function, however, is optimized by playing the game over and over, hence, expected to be inaccurate specially at the early stages of the learning.

SP modifies the rewarding strategy of a given RL problem. For such strategy, it is important to ensure that the modification preserves the optimal policy of the given RL problem \cite{ng1999policy}. To do so, we first introduce some definitions and remarks and then prove that the order of the solutions (policies) is not impacted when SP is used.

\begin{myDefinition}
	\label{def:poilicyindependent}
	\textbf{Policy Independent Reshaping Function (PIRF)}: Let $ R $ a RL environment, defined by $ (S, A, e, r) $, $ r(s, \pi(s)) $ being the reward received at a state $ s $ by taking an action $ \pi(s) $, $ \pi $ being a policy used by a RL agent. Assume we reshape this reward by a function $ \phi(s, \pi(s)) $, $\phi: S \times A \to \mathbb{R}$. A reward reshaping function is an instance of policy independent reshaping function if and only if 
	
	"for any arbitrary $ \pi_1 $ and $ \pi_2 $, $ v(\pi_1, s_0) \le v(\pi_2, s_0) $ is sufficient for $ \hat{v}(\pi_1, s_0) \le \hat{v}(\pi_2, s_0) $"
	
	where $ v(\pi, s_0)=\sum_{i=0}^{n}\gamma^i r(s_i, \pi(s_i)) $ and $ \hat{v}(\pi, s_0)=\sum_{i=0}^{n}\gamma^i \phi(s_i, \pi(s_i)) $, $ s_0 $ is the initial state, $ s_{n+1} $ is a terminal state, and $ \gamma \in [0, 1] $ is the discount factor.
\end{myDefinition}

Essentially, a reward reshaping function is an instance of a \textbf{policy independent reshaping function} (PIRF) if and only if it does not change the order of policies in terms of their total reward under the operator "$\le$". For example, the reward reshaping function $ \phi(s, \pi(s))=r(s, \pi(s))^2+a $ ($ a \in \mathbb{R} $ any arbitrary value) is not an instance of PIRF in general as it may change the order of the solutions when some rewards are smaller than $ a $. The intuition here is that, when this reshaping function is used, the rewards that are smaller than $ a $ decrease while the ones that are larger than $ a $ increase. 

\begin{myRemark}
	\label{rem:policy_ind_optimal}
	Let $ R $ a RL environment defined by $ (S, A, e, r) $ and $ R' $, another RL environment, defined by $ (S, A, e, \phi) $. If $\phi(.)$ is an instance of PIRF then the optimal policy for $ R $ and $ R' $ is the same.
\end{myRemark}
\begin{proof}
	Proof is trivial and can be done by contradiction (assuming the function $\phi$ changes the $ \pi^* $).
\end{proof}

There are multiple conclusions one can draw out of Definition \ref{def:poilicyindependent} and Remark \ref{rem:policy_ind_optimal}. First, a reward reshaping function that is an instance of PIRF preserves optimal and near optimal policies. Hence, an algorithm that is provably able to find an optimal solution for a RL problem (e.g., Q-learning implemented by neural networks under certain conditions \cite{cai2019neural}) would perform similarly if the reward is reshaped by function that is an instance of the PIRF. Second, Definition \ref{def:poilicyindependent} and Remark \ref{rem:policy_ind_optimal} do not make any assumption about the algorithm used for finding the optimal policy. Finally, Definition \ref{def:poilicyindependent} and Remark \ref{rem:policy_ind_optimal} assumed a general case of reshaping a reward using a real-valued function, $\phi: S \times A \to \mathbb{R}$, which enables their broad applicability for reshaping rewards. 

Now we prove that SP is an instance of PIRF, hence, application of SP does not reorder the policies and maintains the optimal policy (as Remark \ref{rem:policy_ind_optimal} stated).
\begin{myTheorem}
	\label{thr:self_punish}
	The SP strategy is an instance of PIRF for any RL problem, defined by $ (S, A, e, r) $, with the initial state $ s_0 $.
\end{myTheorem}
\begin{proof}	
	For simplicity, we define $ r^\pi_i $ the reward the agent receives by taking action $\pi(s_i)$ at state $ s_i $. The total reward received by following arbitrary policies $ \pi_1 $ and $ \pi_2 $ are calculated by $ v(\pi_1, s_0)=r^{\pi_1}_0+\gamma r^{\pi_1}_1...+\gamma^n r^{\pi_1}_n $ and $ v(\pi_2, s_0)=r^{\pi_2}_0+\gamma r^{\pi_2}_1...+\gamma^m r^{\pi_2}_m $, where $ s_{n+1} $ and $ s_{m+1} $ are terminal states. Without loss of generality, we assume that $ v(\pi_1, s_0) \le v(\pi_2, s_0) $. 
	In the case of incorporating the SP strategy, these total rewards are calculated by $ \hat{v}(\pi_1, s_0)=r^{\pi_1}_0+\gamma r^{\pi_1}_1...+\gamma^n r^{\pi_1}_n-p= v(\pi_1, s_0)-p$ and $ \hat{v}(\pi_2, s_0)=r^{\pi_2}_0+\gamma r^{\pi_2}_1...+\gamma^m r^{\pi_2}_m-p= v(\pi_2, s_0)-p$. This indicates that incorporating the SP strategy shifts the total rewards by the same amount, $ p $, for both policies. Hence, if $ v(\pi_1, s_0)\le v(\pi_2, s_0) $ then $ \hat{v}(\pi_1, s_0)\le \hat{v}(\pi_2, s_0) $. As the policies $ \pi_1 $ and $ \pi_2 $ are arbitrary, SP is an instance of PIRF and the proof is complete.
\end{proof}
The assumption here is that there is always a terminal state (good or bad), which means that the episodes are finite (see Section \ref{sec:rlandq}). Based on Remark \ref{rem:policy_ind_optimal} and Theorem \ref{thr:self_punish}, one can conclude that the SP strategy can be used in combination with any RL algorithm that learns the value or an action-value function.  

\subsection{Reward Backfill (RB)}
\label{sec:reward_backfill}
The issue of credit assignment has been observed in training animals, i.e., any time-delay between the action and the reward must be minimal in order for the animal to associate the consequence with the response effectively\cite{pryor1999don,skinner1951teach}. As this may not be possible in many environments, a strategy called the "clicker training" has been used and shown to be effective \cite{pryor1999don,skinner1951teach}. In this strategy, a new reward is shaped which is used to "bridge" the potential time delay of a desirable behavior by the animal and receiving the reward. We mimic a similar strategy to improve DQN. For each state-action with non-zero reward, we propose to backpropagate that reward to previously performed actions in their corresponding states, namely \textit{reward backfill}, RB. This strategy provides information for the learning algorithm on how much an state-action with zero immediate reward may contribute to a future state-action with non-zero reward. We intuitively assume that the actions taken further back contribute less in the reward received for an action in a particular state. Also, we assume that an action with zero-reward contributes to the closest future state-action with non-zero reward only and not further. 

Assume that the policy $ \pi $ is used to generate actions for each given $ s $. Given the initial state $ s_0 $, the policy generates a sequence of actions by which the agent visits $\{ s_1 , s_2 , ...,  s_n \}$, where $ s_{n+1} $ is a terminal state, and receives the rewards $\{ r_0, r_1 , ..., r_n\} $. We propose to modify the reward received at a state $0<i<n$ ($ n $ is the episode length) by 
\begin{equation}
\label{eq:backfill}
\hat{r}_i  = f(i)r_{\mu(i)}
\end{equation}
where $\hat{r}_{i}$ is the estimated reward at the $i^{th}$ state in the states sequence, $\mu(i): \mathbb{N} \to \mathbb{N}$ the index of the closest next state to $ i $ with non-zero reward in the current sequence of states, and $ f: \mathbb{N} \to \mathbb{R} $ a decreasing function. By using this formula, a fraction of the reward at a non-zero reward state next to a given state is backpropagated to modify previous zero rewards. This strategy reflects the impact of actions-states with no rewards in receiving the reward at the state $\mu(i)$. We prove that, under some assumptions, the RB strategy is an instance of PIRF (see Definition \ref{def:poilicyindependent}). To prove that, we first define the sparsity length in an episode as follows.
\begin{myDefinition}
	\label{def:sparsitylength}
	Let $ s_0 $ an initial state, $ \pi $ a policy, $ j^\pi_k \in \{0, 1, ..., n\} $  the index of the $ k^{th} $ non-zero reward received by following $ \pi $, $ k \in \{0, 1, ..., M^\pi\} $, $M^\pi+1$ the number of non-zero rewards the agent receives in the episode by following policy $\pi$. We define $ l^\pi_i = |j^\pi_{i-1} - j^\pi_{i}|$ as the \textbf{$ i^{th} $ sparsity length} by following the policy $\pi$, $ i \in \{1, ..., M^\pi\} $. 
\end{myDefinition}

The definition of sparsity length depends on the sampling frequency of the environment states. If for example the vision data is used as samples then the number of frames between two frames in which the agent receives a reward defines the sparsity length (see section \ref{sec:sparsitylength} for details).

\begin{myTheorem}
	\label{thr:sparsitylength}
	Let $ s_0 $ an initial state of a RL problem, $ R $, defined by $ (S, A, e, r) $. If $ \sum_{i=0}^{l^\pi_t}f(i) $ is a constant for all $ t $ and $\pi$ then RB strategy (Eq. \ref{eq:backfill}) is an instance of PIRF.
\end{myTheorem}
\begin{proof}
	For simplicity, we define $ r^\pi_i $ the reward the agent receives by taking action $\pi(s_i)$ at state $ s_i $. Let $\pi_1$ and $\pi_2$ two arbitrary policies, taking the agent from state $ s_0 $ to states $ s_n $ and $ s_m $, respectively, $ N $ and $ M $ the number of non-zero rewards received by following those policies, $ v(\pi_1, s_0) $ and $ v(\pi_2, s_0) $ the total reward received by following $ \pi_1 $ and $\pi_2$. One can write $ v(\pi_1, s_0)= h(\pi_1, 0)+h(\pi_1, 1)...+h(\pi_1, N) $ where $  h(\pi, i)=\gamma^{j^{\pi}_i} r^{\pi}_{j^{\pi}_i} $, as all other rewards are zero (see Definition \ref{def:sparsitylength}). The same can be done for $ v(\pi_2, s_0) $. By using the RB, the total reward received by the agent following the policy $\pi_1$ is modified to
	\begin{equation}
	\label{eq:vhat}
	\hat{v}(\pi_1, s_0)=h(\pi_1, 0)\sum_{k=0}^{l^{\pi_1}_1}f(k)+...+h(\pi_1, N)\sum_{k=0}^{l^{\pi_1}_N}f(k)
	\end{equation}
	This can be also written for $ \hat{v}(\pi_2, s_0) $. 
	Assume $ \sum_{k=0}^{l^{\pi}_t}f(k)=z $ for all $ t $ and $\pi$, $ z \in \mathbb{R}^+ $ a constant. In that case, $ \hat{v}(\pi_1, s_0)=v(\pi_1, s_0)z $ and $ \hat{v}(\pi_2, s_0)=v(\pi_2, s_0)z $. Hence, if $ v(\pi_1, s_0) \le v(\pi_2, s_0) $ then $ v(\pi_1, s_0)z \le v(\pi_2, s_0)z $ ($ z>0 $), which means $ \hat{v}(\pi_1, s_0) \le \hat{v}(\pi_2, s_0) $. Hence, a sufficient condition to ensure RB is an instance of PIRF is that $ \sum_{i=0}^{l^\pi_t}f(i) $ is a constant for all $ t $ and $\pi$, which completes the proof.
\end{proof}

We investigate a choice for the function $ f $ as $ f(i)=\lambda^{i} $, where $ \lambda \in [0,1) $. With this setting and for a given environment,  $\sum_{i=0}^{l^\pi_t}f(i)=\sum_{i=0}^{l^\pi_t}\lambda^i=\frac{1-\lambda^{1+l^\pi_t}}{1-\lambda}$ for any $ \pi $ and $ t $. In implementation, $ \lambda^i $ becomes infinitesimally small when $ i $ is large and would be considered as zero because of the floating point precision. In practice, for this choice of $ f $,  we back-propagate the rewards for a minimum sparsity length, $ l_{min} $, i.e., $ f(i)=\lambda^{i} $ if $ i<l_{min} $, otherwise, $ f(i)=0 $. This would lead to $ \sum_{i=0}^{l^\pi_t}f(i) $ being equal to $ \frac{1-\lambda^{1+l_{min}}}{1-\lambda} $, as long as $ l^\pi_t>l_{min} $ for all $ t $ and $ \pi $. For environments with large sparsity length ($ l^\pi_t $ is large), $ l_{min} $ could be ignored in practice because $ \lambda^{1+l^\pi_t} $ is infinitesimaly small, making $ \frac{1-\lambda^{1+l^\pi_t}}{1-\lambda} $ almost equal to the constant $ \frac{1}{1-\lambda} $ for all $ \pi $ and $ t $. For example, $ \lambda^{1+l} $ is almost $ 2.1e-5 $ for an environment with $ l=25 $ and $ \lambda=0.65 $, and reduces down to almost $ 4.46e-10 $ for $ l=50 $. On the other hand, for environments with frequent rewards (small sparsity length), one may need to decrease the value of $\lambda$ to make sure $ \lambda^{1+l} $ is small or ideally zero. Note, however, that if the sparsity length in an environment is small then the problem of sparse rewards is not present, dimming the need for the RB strategy. See sections \ref{sec:sparsitylength} and \ref{sec:conclusion} for further discussions and examples.

Both SP and RB strategies can be implemented as a part of the procedure which is used to select instances from the replay memory to form the training batch (aka, minibatch). For each instance selected from the replay memory for training, the reward is replaced by the value of $p$ if the next state (i.e., $s'_m$, where $m$ is the index of the memory) in that instance is terminal. If the state is not terminal then Eq. \ref{eq:backfill} (RB strategy) is used to estimate the reward. For an efficient implementation, one can store the number of states for which the agent has not received any reward and use that to back-calculate the $\hat{r}_i$ in $ O(1) $.

\section{Experiments}
\label{sec:experiment}
In this section we first investigate the impact of the value of $ p $ and $\lambda$ on D2QN and then compare the final results using a hybrid RB and SP strategy.
\subsection{Test Environment and Experiment Settings} 
We used 30 Atari 2600 games for our comparisons, including UpNDown, BankHeist, MsPacman, Qbert, Zaxxon, Alien, Amidar, Tutankham, AirRaid, Kangaroo, Jamesbond, Gravitar, Seaquest, Hero, WizardOfWor, Frostbite, Venture, Centipede, Freeway, Berzerk, RoadRunner, Carnival, Asterix, Solaris, SpaceInvaders, KungFuMaster, Assault, Krull, Riverraid, and Breakout. The deterministic version of all games were used with 4 frames action repetition. We used Python 3 (source codes and detailed analyses are available as supplementary material) to implement DQN variants with the discount factor ($\gamma$) of 0.99, learning rate of 0.0001, learning method of RMSprop, maximum number of training episode of 4,000, replay memory size of 1M, memory update of every 4 steps, minibatch size of 32, initial epsilon of 1.0 that decays to 0.1 over 100K steps after the exploration step (first 50K), Huber as the loss function, and maximum episode length of 18K frames. The number of steps between target network updates was 10,000 when double targeting strategy was used. In terms of hardware, we used a GPU cluster, each node equipped with two NVIDIA GPU Volta V100 and two skylake Intel Xeon 6132 processor, and over 300GB of RAM. We allocated 4 cores, one GPU, and 30GB of RAM to each run. We run all tests for 4000 episodes unless explicitly specified otherwise.

\subsection{Evaluation Metrics}
We measured the \textit{performance} of each method for each environment by averaging the total reward the agent received in the last 100 episodes during the training. The use of this measure allows evaluating the number of times a method would need to play the game to learn it. We used three evaluation matrices for comparisons: \textit{average rank}, \textit{average improvement percentage}, and \textit{percentage of improved games}. The \textit{average rank} provides the ranking of different methods according to the performance measure. To calculate the average rank for each environment, we ranked the results of different settings (including the original algorithm) by sorting their performances descendingly and then averaged the ranks across all environments. This gives us an indication of how would each setting perform in comparison with others across all environments. The \textit{average improvement percentage} measures the average percentage of improvement over the original algorithm across all environments on the performance measure. We calculated the average percentage of performance improvement, where each strategy with different settings was compared against the original version (the improvement is $100(P^o-P^s)/P^o$ where $P^o$ is the performance of the original version of the algorithm and $P^s$ is the performance of the algorithm with SP or RB strategies incorporated in it). 

\subsection{Sparsity length}
\label{sec:sparsitylength}
It was discussed in section \ref{sec:reward_backfill} that the sparsity length in environments has a direct impact on the efficiency of RB strategy. Figure \ref{fig:lambda_converge} shows the mean sparsity length of 30 Atari games when D3QN was applied to find the optimal policy. It is seen that this length is larger than 25 in most of the environments tested (23 over 30 tested), which was used in our experiments. In the cases where the sparsity length is small there would be a need for a smaller $\lambda$ to compensate. 

\begin{figure}
	\includegraphics[width=0.35\textwidth]{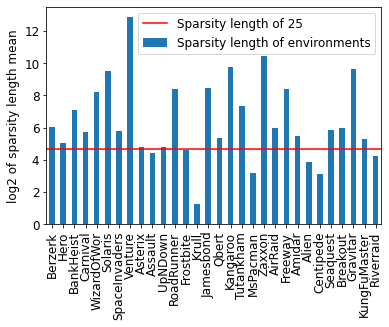}
	\centering	
	\caption{Sparsity length for 30 Atari games when D3QN was used. 23 out of 30 environments had a sparsity length of 25 or larger.}
	\label{fig:lambda_converge}
\end{figure}

One should note that the RB approach is not expensive (done in $ O(1) $) as it backpropagates reward signals and does not need any forward calculation of a neural network. 

\subsection{Parameters Setting}
We incorporated SP and RB strategies to D2QN and tested the performance of the algorithm for $ p \in \{1, 10, 50, 100, 200\} $, $\lambda \in \{.15, .65, .75, .85, .95\}$  and environments UpDown, Carnival, Gravitar, MsPacman, Qbert, Spaceinvaders, Berzerk, and Breakout. Results for SP settings are reported in Figure \ref{fig:SP_setting}. Figure \ref{fig:SP_setting} (a) shows the average ranks (the smaller, the better) of D2QN with and without SP and different values for $ p $. The D2QN without SP in most episodes ranked worse than D2QN with SP ("original" bar in the graph), demonstrating the effectiveness of this strategy. From sub-figure (a), it is clear that $ p=1 $ and $ p=10 $ provide the best average rank across all tested environments. Sub-figure (b) indicates that the average improvement is maximized across tested environments for $ p=1 $.
\begin{figure}[t]
	\begin{tabular}{c}		
		\includegraphics[width=0.45\textwidth]{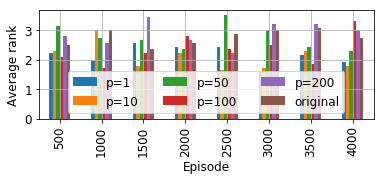}\\(a)\\\includegraphics[width=0.45\textwidth]{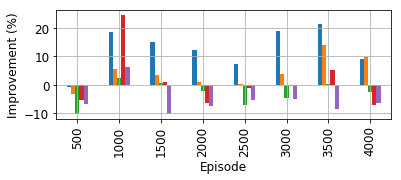}\\(b)
	\end{tabular}
	\centering	
	\caption{Positive punishment test on D2QN over 4000 episodes. (a) Average rank of different $ p $ values, (b) average improvement percentage across all 8 environments.}
	\label{fig:SP_setting}
\end{figure}

Figure~\ref{fig:RB_setting} shows the results of incorporating RB to D2QN tested on the same set of environments. The sub-figures (a) and (b) show that, for some values of $ \lambda $, RB improves D2QN (in terms of average rank and improvement percentage). The parameter $\lambda$ decides the amount of the reward propagated backwards, with a larger value of $\lambda$ leads to a larger portion of the reward back propagated. In these sub-figures it is seen that the maximum improvement and best rank takes place when $\lambda=.65$ across all tested settings. 
\begin{figure}[h]
	\begin{tabular}{c}		
		\includegraphics[width=0.45\textwidth]{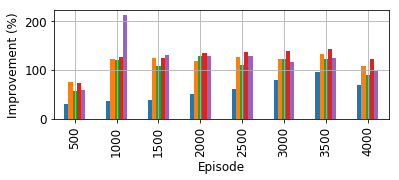}\\(a)\\\includegraphics[width=0.45\textwidth]{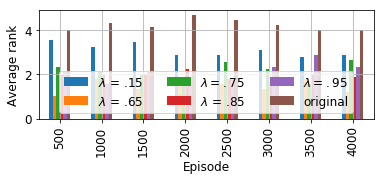}\\(b)\\ 
	\end{tabular}
	\centering	
	\caption{Reward backfill test on D2QN for different values of $\lambda$. (a) and (b) are similar to what was described in Fig. \ref{fig:SP_setting}.}
	\label{fig:RB_setting}
\end{figure}

\subsection{Comparison results}
Three different deep Q-learning algorithms were tested as shown in Figure~\ref{fig:d2qn_all}(a) and (b) when SP and RB were incorporated to the base methods ($ p=1 $ and $\lambda=0.65$). It is seen that the proposed strategies could improve the DQN methods in over 26 out of 30 tested environments by up to $ 2,800\% $ in some cases. The improvement for D2QN and D3QN was in 19 out of 30 games for up to $ 1,900\% $ and $ 700\% $, respectively. In some cases, however, the proposed strategies lead to a performance drop up to $ 50\% $ for all methods. These results well demonstrate the effectiveness of our proposed strategies. 

In addition, Figure~\ref{fig:d2qn_all}(b) shows performance comparison between our proposed strategies, Lyapanov reward shaping strategy \cite{dong2020principled}, K-nn based trajectory selection shaping strategy \cite{zhu2020meta}, and a noisy neural network structure used in Q-learning \cite{fortunato2017noisy}. The figure shows that the proposed reward shaping strategies could improve D2Qn more effectively in comparison to other shaping strategies tested. Noisy structure improved the performance of D2QN in 27 out of 30 Atari games comparing to 19 out of 30 when the proposed strategies were used. This shows that the Noisy D2QN could provide a better baseline to benefit from our proposed shaping strategies, which is left as a future work.

\begin{figure*}[t]
	\begin{tabular}{c}
		\includegraphics[width=0.8\textwidth]{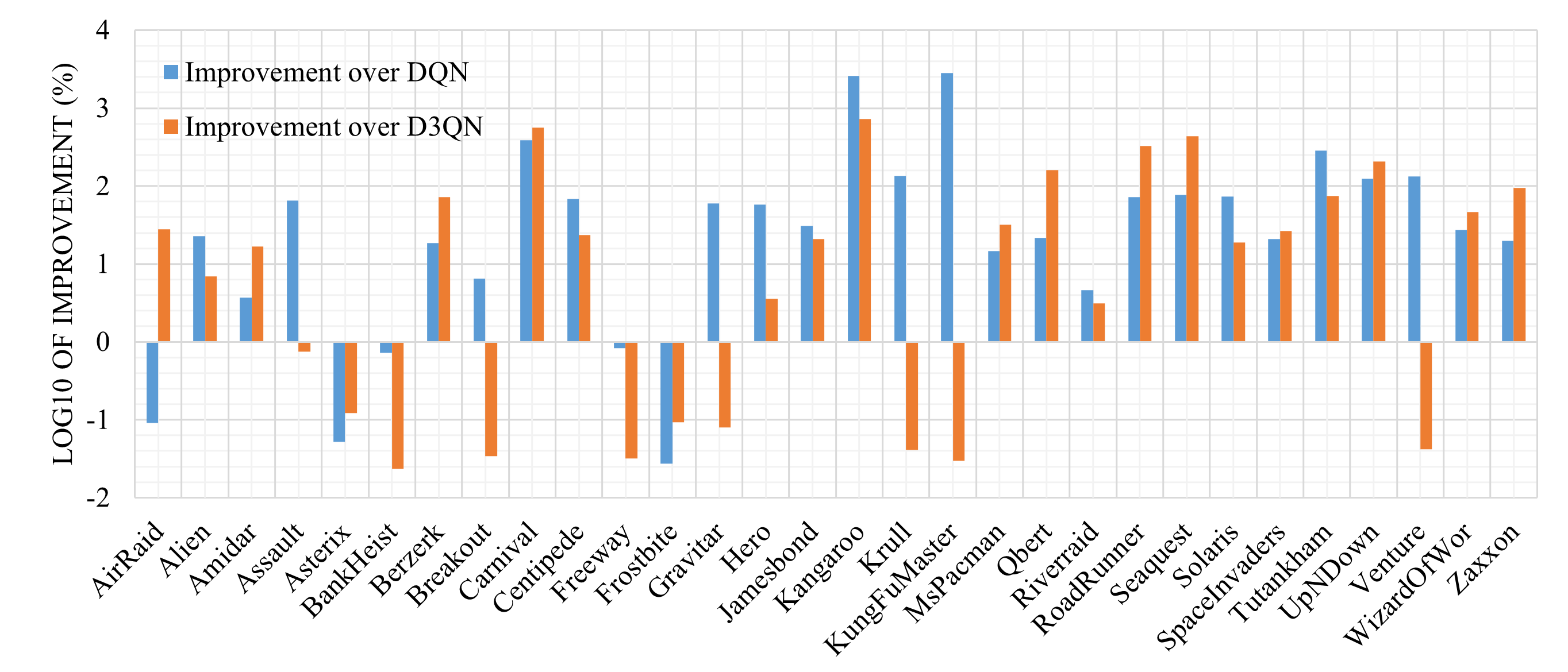}\\
		(a)\\
		\includegraphics[width=0.8\textwidth]{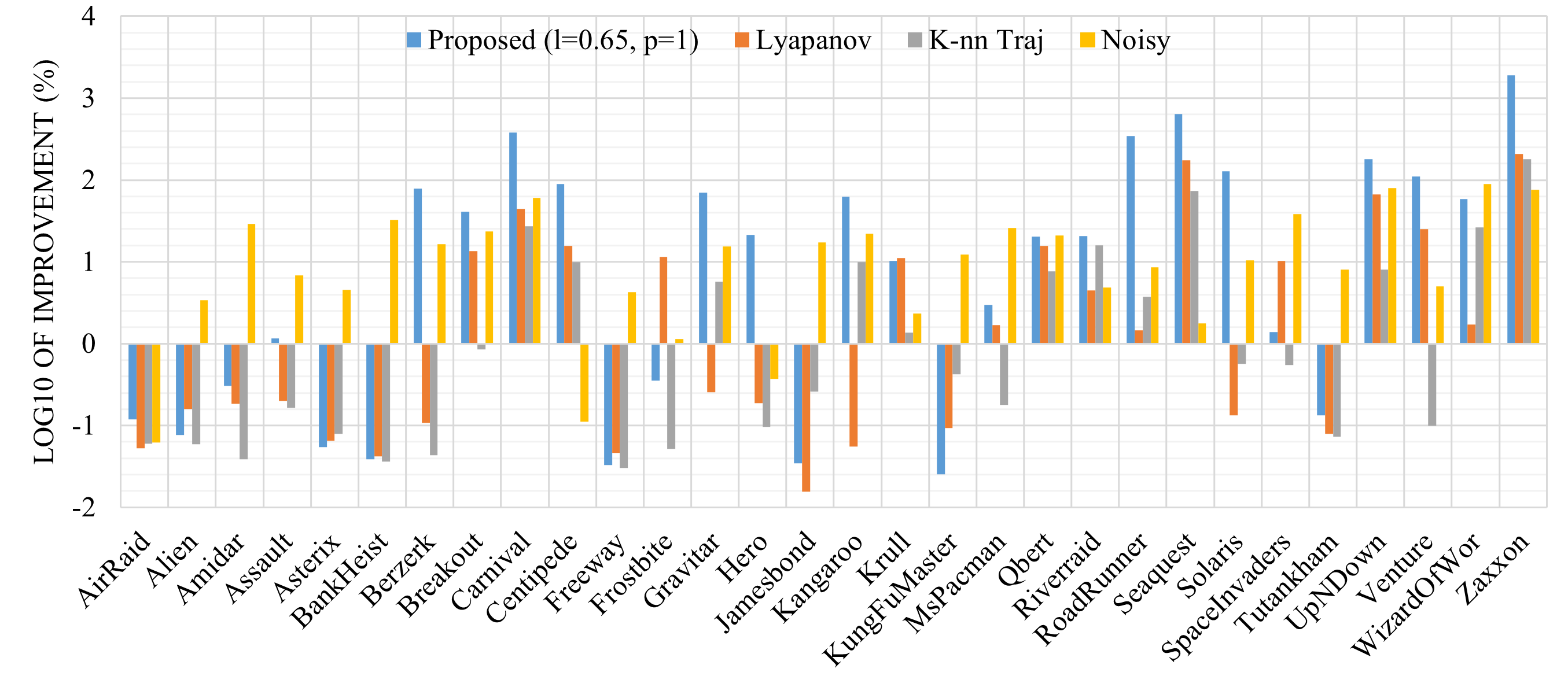}\\
		(b)\\
	\end{tabular}	
	\centering	
	\caption{Performance improvement over original (a) DQN and D3QN, (b) D2QN. (b) also indicates the performance improvement over D2QN for the Lyapanov shaping method \cite{dong2020principled}, K-nn trajectory shaping \cite{zhu2020meta}, and Noisy DQN \cite{fortunato2017noisy}. Proposed method used SP with $ p=1 $ and RB with $\lambda=0.65$.}
	\label{fig:d2qn_all}
\end{figure*}

\section{Discussion and Future Works}
\label{sec:conclusion}

In this paper, we introduced two strategies, called the self punishment and the reward backfill, to deal with the credit assignment problem. The strategies were inspired by the psychology of conditioning and behaviour shaping, namely the operant conditioning and clicker training. Self punishment signals the agent mistakes to enable quicker learning to avoid such mistakes. The reward backfill propagates the value of a rewarded state back to the previous actions to signal the contribution of those previous actions to the action which explicitly rewarded. We proved both of these two strategies maintain the order of policies in the space of all possible policies in terms of their total reward, and, by extension, maintain the optimal policy. Hence, our proposed strategies integrate with any reinforcement learning algorithm that learns a value or action-value function through experience. We further showed how to effectively implement these strategies to achieve a $ O(1) $ complexity in their calculations. Finally, we experimentally showed that these strategies can improve the ability of Deep Q-learning methods in 30 Atari 2600 games. 

Both self punishment and reward backfill methods should be viewed as a supplementary strategy to the family of deep Q-learning algorithms, which are designed for overcoming the credit assignment problem. Because these strategies do not modify the internal computational mechanism of a deep Q-learning algorithm, they are also applicable in conjunction with other algorithms developed for overcoming sparse reward problem, such as Prioritized Experience Replay~\cite{schaul2015prioritized} or Curiosity-driven Exploration~\cite{pathak2017curiosity} mentioned in Section \ref{sec:background}. 

\textbf{For the self punishment strategy}, we realized that the improvement is not consistent across all environments, i.e., some environments are improved and some not. One potential reason is that some environments have more complicated terminal states where an agent reaching the terminal states is not directly caused by the latest actions it takes but a longer term strategy. In MsPacman, for example, an agent walks into a state where the enemies come from both sides, where the agent is going to reach the terminal state regardless of the action it takes. Another potential reason, also observable in our results, is that the optimal value for $ p $ can be different across environments, which encourages an adaptive strategy to control the value of $ p $ automatically for different iterations and environments. Finally, we observed that a larger $ p $ would not improve the results. One potential reason behind this event is that larger $ p $ potentially generates larger gradient value, which may lead to delay in convergence or even divergence. Hence, based on our experiments, one simple criterion to make this strategy beneficial is to avoid using large $ p $ values, which again could be adjusted given the environment.

\textbf{For the reward backfill strategy}, we observed that the best $\lambda$ is sensitive to the sparsity length in the environments, as it was predicted by Theorem \ref{thr:sparsitylength}. For example, in the game MsPackman, the agent receives rewards very frequently. Hence, according to the Theorem \ref{thr:sparsitylength}, there is no guarantee that the incorporation of RB would not change the best solution of the RL problem. We actually observed that $\lambda=0.95, 0.85$ leads to the worst results in MsPackman, indicating the negative impact the RB strategy may have on the learning when the sparsity length is short. Note, however, an environment with frequent rewards would not have the problem of sparse rewards, dimming the need for the RB strategy. We also noticed that the best value for $ \lambda $ is environment dependent. Hence, one criterion to successfully use this approach is to ensure the sparsity length is long enough or reduce the value of $ \lambda $ (or $ l_{min} $) to ensure consistency in the order of the optimality of policies.

The impact of the proposed methods appears to be environment dependent, evidenced by our results. Thus, a natural extension to our work is to design an adaptive approach to adjust relevant parameters of the methods during the learning procedure in a way that the addition of SP or RB is more likely to be beneficial. One example is to leverage the fact that the impact of RB is directly related to the value of $ \lambda $ and sparsity of the rewards provided by the environment.

One interesting result is that there is a huge improvement (about 10 times more than others) on the original DQN algorithm. A known issue with DQN is the overestimation of Q-value due to the greedy policy of approximation, and D2QN and D3QN subsequently mitigate the problem by applying additional weights and estimators. Although the reason behind such big improvement in DQN is unclear, it would be worthwhile to investigate whether the PP and RB strategies can help the overestimation issue in DQN.

\subsubsection*{Acknowledgments}

The authors would like to acknowledge the use of GPU cluster, Wiener, at Queensland Brain Institute, the University of Queensland, for running the experiments. They would also like to thank Jake Carroll who assisted in optimizing the codes to make the best use of that cluster.

\medskip

\small

\bibliographystyle{IEEEtran}
\bibliography{reference}
\end{document}